\documentclass{article}

\usepackage[nonatbib,preprint]{neurips_2019}

\usepackage[numbers]{natbib}

\usepackage{microtype}
\usepackage{graphicx}
\usepackage{subcaption}
\usepackage{booktabs} 
\usepackage{hyperref}

\usepackage{amsmath}
\usepackage{amssymb}
\usepackage{mathtools}
\usepackage{amsthm}
\usepackage{enumitem}
\usepackage{multirow}

\theoremstyle{plain}
\newtheorem{theorem}{Theorem}[section]
\newtheorem{proposition}[theorem]{Proposition}

\theoremstyle{definition}

\theoremstyle{remark}
\newtheorem{remark}[theorem]{Remark}

\providecommand{\Paren}[1]{\ensuremath{\left( #1 \right)}}
\providecommand{\Braces}[1]{\ensuremath{\left\{ #1 \right\}}}
\providecommand{\norm}[1]{\lVert#1\rVert}
\providecommand{\Norm}[1]{\ensuremath{\left\lVert#1\right\rVert}}
\providecommand{\defeq}[0]{\stackrel{\text{def}}{=}}

\def\tr{^\top}

\providecommand{\EE}[2]{\ensuremath{\mathbb{E}_{#1}\left[ #2 \right]}}

\providecommand{\partfrac}[2]{\ensuremath{\tfrac{\partial #1}{\partial #2}}}
\providecommand{\cpm}{\!\pm\!}

\begin{document}

\title{Strengthening the Internal Adversarial Robustness in Lifted Neural Networks}

\author{%
  Christopher Zach  \\
  Chalmers University of Technology\\
  Gothenburg, Sweden \\
  \texttt{zach@chalmers.se}
}

\maketitle

\begin{abstract}
  Lifted neural networks (i.e.\ neural architectures explicitly optimizing
  over respective network potentials to determine the neural activities) can
  be combined with a type of adversarial training to gain robustness for
  internal as well as input layers, in addition to improved generalization
  performance. In this work we first investigate how adversarial robustness in
  this framework can be further strengthened by solely modifying the training
  loss. In a second step we fix some remaining limitations and arrive at a
  novel training loss for lifted neural networks, that combines targeted and
  untargeted adversarial perturbations.

\end{abstract}

\section{Introduction}

Lifted neural networks
(e.g.~\cite{carreira2014distributed,zhang2017convergent,li2019}) and the
related inference learning methods~\cite{song2020can} cast the inference of
neural activities for a given data sample as explicit optimization task,
instead of using prescribed rules like the standard forward pass in layered
neural networks. Lifted neural networks are thought to be more plausible to
drive synaptic plasticity and learning in the
brain~\cite{friston2005,salvatori2023brain,song2024inferring}, but these
methods are also relevant for certain types of analog, energy-based hardware
used in neuromorphic computing~\cite{kendall2020,scellier2023energybased}. The
usual setup for learning in lifted neural networks leads to a joint
minimization task w.r.t.\ network weights and neural activities. The ability
to trained such networks by e.g.\ joint minimization over weights and neural
states, or by respective alternating optimization, has been one motivating
aspect for lifted neural
networks~\cite{carreira2014distributed,zhang2017convergent}). Recently, a less
natural min-max formulation for lifted networks has been shown to be
beneficial in theory and practice due to its connection with adversarial
training. Unlike standard adversarial training, the min-max formulation for
lifted neural networks incorporates adversarial perturbations at all layers in
the network, not just input layer perturbations.

In this work we show how to potentially strengthen adversarial training for
lifted networks and how to avoid certain shortcomings. A first contribution
relaxes the min-max formulation to render the training objective more
difficult (Section~\ref{sec:generalizing_AROVR}). A second contribution
incorporates targeted attacks into the framework, which addresses some still
existent shortcomings (Section~\ref{sec:targeted}). We intentionally restrict
the setting to small training sets, such that trained networks easily operate
in the interpolation regime as our main interest at this point is to leverage
the training loss to improve a network's generalization ability.

\section{Background}
\label{sec:background}

Since lifted neural networks and their realizations are not widely known, we
provide in this a short introduction. In this work we focus on layered
feed-forward networks, where the neural activities for a given input sample
$x$ satisfy $z_0=x$ and $z_k=f_k(W_{k-1}z_{k-1})$ for $k=1,\dotsc,L$. Thus,
the network has $L$ layers and each layer has an activation function
$f_k$. For notational brevity we omit the biases. All theoretical results
should generalize to arbitrary feed-forward networks (but with a signficantly
more cumbersome notation).

\subsection{Network Potentials}

The general idea behind lifted neural networks is convert the inference step
to determine the neural activities explicitly into an optimization task. The
neural activities $z^*$ for a given input $x$ are the solution to minimizing a
network potential $U_\theta(z;x)$ for the network parameters $\theta$, i.e.\
$z^*=\arg\min_z U_\theta(z;x)$. This network potential can be based on
energy-based models and be influenced by neurosciene (where Hopfield
networks~\cite{amari1972learning,hopfield1982neural} and Boltzmann machines~\cite{ackley1985learning}
are the most prominent examples), but for layered feed-forward neural networks
the potential $U_\theta$ is typically penalizing deviations from the desired
forward dynamics between two layers, $z_k=f_k(W_{k-1}z_{k-1})$. The easiest
penalizer is the quadratic one, leading to a network potential of the form
\begin{align}
  U_\theta(z;x) = \tfrac{1}{2} \sum\nolimits_k \norm{z_k - f_k(W_{k-1}z_{k-1})}^2,
\end{align}
with $z_0=x$. The least-squares terms can be optionally reweighted. Predictive
coding networks~\cite{whittington2017approximation,salvatori2023brain}, and
the method of auxiliary
coordinates~\cite{carreira2014distributed,choromanska2019beyond} are based on
such quadratic penalizers. A different type of network potential is based on
the Fenchel-Young inequality from convex analysis, which also has been used as
a general loss function in machine learning~\cite{blondel2020learning}. Let
$G_k$ be a convex function and $G_k^*$ its convex conjugate, then we introduce
the ``Fenchel-Young divergence''
\begin{align}
  D_k(z_k \| a_k) := G_k(z_k) - z_k\tr a_k + G_k^*(a_k),
\end{align}
which is always non-negative due to the Fenchel-Young inequality. With
$f_k = \partial G_k^*$ (and therefore $f_k^{-1}=\partial G_k$) it is known that
$D_k(z_k \| a_k) = 0$ iff $z_k = f_k(a_k)$.
Therefore a suitable network potential is given by
\begin{align}
  U_\theta(z;x) \!=\! \sum\nolimits_k \!\!D_k(z_k\|W_{k-1}z_{k-1})
  \!=\! \sum\nolimits_k \!\!\Paren{ G_k(z_k) \!-\! z_k\tr W_{k-1} z_{k-1} \!+\! G_k^*(W_{k-1} z_{k-1}) }\!,
  \label{eq:LPOM}
\end{align}
which forms the basis of ``lifted proximal operator
machines''~\cite{li2019,li2020}. The most prominent choice for $G_k$ is given
by $G_k(z_k)=\norm{z_k}^2/2 + \imath_{C_k}(z_k)$, where $C_k$ is a convex set
and $\imath_{C_k}(z_k)=0$ iff $z_k\in C_k$ and $+\infty$ otherwise. The
corresponding activation function $f_k$ is given by
$f_k(a_k) = \Pi_{C_k}(a_k)$, i.e.\ the (Euclidean) projection of $a_k$ into
the convex set $C_k$. The most common choices for $C_k$ are
$C_k=\mathbb{R}$ and $C_k=\mathbb{R}_{\ge 0}$, which yield linear and ReLU
activation functions, respectively.

All these network potentials can be understood as relaxations of deeply nested
optimization tasks~\cite{zach2021}. Observe that $U_\theta(z;x)\ge 0$ and zero
iff $z_k=f_k(W_{k-1}z_{k-1})$ follows the intended forward
dynamics. \eqref{eq:LPOM} can be generalized to arbitrary feed-forward
computational graphs (such as residual connections) by using the following
network potential,
\begin{align}
  U_\theta(z;x) = G(z) - z\tr (Wz + W_0x) + G^*(Wz + W_0x),
\end{align}
where $W$ is a strictly lower-triangular matrix (i.e.\ with zeros on the
diagonal) and $W_0$ is a matrix injecting the input $x$ into the network.

\subsection{Training of Lifted Neural Networks}

For a target loss $\ell$, the underlying task of learning in lifted neural
networks is given by
\begin{align}
  \min\nolimits_\theta \ell(z^*)
  \qquad \text{s.t. } z^* = \arg\min\nolimits_z U_\theta(z;x) \iff U_\theta(z^*;x) \le \underbrace{\min\nolimits_z U_\theta(z;x)}_{= 0}.
  \label{eq:learning}
\end{align}
The true label (or regression target) is absorbed into $\ell$, and for
notational brevity we consider single-sample training sets (or think of
$x$ and $z$ as the entire training input and neural activities,
respectively). Implicit differentiation of~\eqref{eq:learning} yields error
back-propagation, but we are interested in a different route. By fixing an
(inverse) Lagrange multiplier $1/\beta>0$ for the inequality constraint
in~\eqref{eq:learning} we obtain a \emph{relaxed optimal value reformulation}
(ROVR,~\cite{zach2021}),
\begin{align}
  \min\nolimits_\theta \min\nolimits_z \ell(z) + \tfrac{1}{\beta} U_\theta(z;x). \tag{ROVR${}_\beta$}
  \label{eq:ROVR}
\end{align}
Many training methods for lifted neural networks (and certain energy-based
models~\cite{scellier2017}) are implicitly based on~\eqref{eq:ROVR}, and
equivalence with back-propagation (under mild conditions) is readily derived
when $\beta\to 0^+$. \eqref{eq:ROVR} is a sensible training objective in its
own right, since it aims to lower both the target loss $\ell$ and deviations
from the neural dynamics $U_\theta$. Interestingly, the following
\emph{adversarial ROVR} (AROVR,~\cite{hoier2024two}) based on adversarial robustness,
\begin{align}
  \min\nolimits_\theta \max\nolimits_z \ell(z) - \tfrac{1}{\beta} U_\theta(z;x). \tag{AROVR${}_\beta$},
  \label{eq:AROVR}
\end{align}
turns out to be a viable training objective as well (with small generalization
benefits compared to ROVR~\cite{laborieux2021scaling,hoier2024two}). The AROVR
will be discussed in more detail in Section~\ref{sec:generalizing_AROVR}. In
practice, training of lifted networks is performed by first (approximately)
inferring the neural activities, i.e.\ solving the objective w.r.t.\
$z$, followed by a (stochastic) gradient-based update of $\theta$.

If $\min_z U_\theta(z;x)$ is not just a constant for all $\theta$, then the
ROVR and AROVR correspond to contrastive learning tasks,
\begin{align}
  \begin{split}
    &\min\nolimits_\theta \Braces{ \min\nolimits_z \Braces{ \ell(z) + \tfrac{1}{\beta} U_\theta(z;x) } - \tfrac{1}{\beta} \min\nolimits_z U_\theta(z;x) } \\
    &\min\nolimits_\theta \Braces{ \tfrac{1}{\beta} \min\nolimits_z U_\theta(z;x) - \min\nolimits_z \Braces{ -\ell(z) + \tfrac{1}{\beta} U_\theta(z;x) } },
  \end{split}
\end{align}
which have realizations in the literature as contrastive Hebbian
learning~\cite{movellan1991,xie2003,zach2019} and equilibrium
propagation~\cite{scellier2017,laborieux2021scaling}.

\section{Understanding and Generalizing the AROVR}
\label{sec:generalizing_AROVR}

In this section we propose a generalization of the AROVR and at the same time
make a clearer connection with adverserial robustness. Let
$U_\theta(z;x)$ be a network potential such that
$z^*(x) = \arg\min_z U_\theta(z;x)$ yields the desired network dynamics
$z^*(x)=f_\theta(x)$. We also assume that $U_\theta(\cdot;x)$ is coercive for
all $\theta$ and $x$, i.e.\
$\lim_{z:\norm{z}\to\infty} U_\theta(z;x) \to \infty$. This implies that the
sublevel set $S_\delta(x,\theta) := \{ z: U_\theta(z;x) \le \delta\}$ (for a
$\delta>0$) is bounded. Now consider the following adversarial training
problem,
\begin{align}
  \label{eq:simple_AT}
  \min\nolimits_\theta \max\nolimits_z \ell(z) \quad \text{s.t. } z \in S_\delta(x,\theta)
  & &\equiv & & \min\nolimits_\theta \max\nolimits_z \ell(z) \quad \text{s.t. } U_\theta(z;x) \le \delta.
\end{align}
$\delta>0$ plays a similar role as the radius of allowed perturbations in
norm-based adversarial attacks. If $U_\theta(z;x)$ is 1-strongly convex in
$z$, then $S_\delta(x,\theta)$ is contained a the
$\sqrt{2\delta}$-norm ball centered at $z^*(x)$. Using a fixed multiplier
$\beta^{-1}$ for a $\beta>0$ to convert the constraint into a penalizer yields
\begin{align}
  \label{eq:relaxed_simple_AT}
  \min\nolimits_\theta \max\nolimits_z \ell(z) - \tfrac{1}{\beta} U_\theta(z;x),
\end{align}
which is exactly the AROVR. We now generalize~\eqref{eq:simple_AT} to
incorporate a distance-like mapping $d_\gamma(z,z^*(x))$ instead of the
difference of potentials,
$U_\theta(z;x) - U_\theta(z^*(x))=U_\theta(z;x)$. We assume that
$d_\gamma(z,z)=0$ and $d_\gamma(z,z')$ is coercive for all $z$, $z'$ and
$\gamma$. A generalization of~\eqref{eq:simple_AT} is now given by
\begin{align}
  \label{eq:generalized_AT}
  \min\nolimits_\theta \max\nolimits_{z^-} \ell(z^-) \qquad \text{s.t. } d_\gamma(z^-,z^*) \le \delta \qquad z^* = \arg\min\nolimits_z U_\theta(z;x) .
\end{align}
The adversarial region is now shaped by $d_\gamma$, which can be completely
independent of the network potential $U_\theta$. The Lagrangian relaxation
of~\eqref{eq:generalized_AT} using fixed multipliers $\lambda/\beta>0$ and
$1/\beta>0$ yields
\begin{align}
  \label{eq:relaxed_gen_AT}
  \min\nolimits_\theta \max\nolimits_{z^-,z^*} \ell(z^-) - \tfrac{\lambda}{\beta} d_\gamma(z^-,z^*) - \tfrac{1}{\beta} U_\theta(z^*;x) .
\end{align}
We consider $\lambda$ as part of the parameter vector $\gamma$ in the following.

At this point we should emphasize two major differences
of~\eqref{eq:generalized_AT} and a standard adversarial training problem:
\begin{enumerate}
\item In~\eqref{eq:generalized_AT}, adversarial perturbations can be applied
  everywhere in the network---not just to the network's input units.
\item The adversarial region defined by
  $\{z: d_\gamma(z,z^*)\le \delta\}$ can be highly anisotropic, and therefore
  different units in the network can have different magnitudes of adversarial
  perturbations.
\end{enumerate}
Consequently, a worthwhile research question is how to choose $\gamma$ in
order to minimize the generalization error of an adversarially trained network.

\begin{proposition}
  \label{prop:gen_AROVR}
  Adversarial training based on~\eqref{eq:relaxed_gen_AT} is harder than
  training via AROVR~\eqref{eq:relaxed_simple_AT} in the following sense:
  \begin{align}
    \label{eq:prop1}
    \max_{z^-} \ell(z^-) - \tfrac{1}{\beta} U_\theta(z^-;x)
    \le \max_{z^-,z^*} \ell(z^-) - \tfrac{\lambda}{\beta} d_\gamma(z^-,z^*) - \tfrac{1}{\beta} U_\theta(z^*;x) .
  \end{align}
\end{proposition}
\begin{proof}
  The left-hand side of~\eqref{eq:prop1} is exactly the r.h.s.\ with the
  additional constraint $z^-=z^*$, and the inequality is obtained by noting
  that $d_\gamma(z^-,z^-)=0$ by assumption. \qed
\end{proof}
The parameters $\gamma$ are ideally obtained via meta-learning, but can one
also optimize jointly over $\theta$ and $\gamma$? Consider
\begin{align}
  \min\nolimits_{\theta,\gamma} \max\nolimits_{z^-,z^*} \ell(z^-) - \tfrac{1}{\beta} d_\gamma(z^-,z^*) - \tfrac{1}{\beta} U_\theta(z^*;x) \nonumber.
\end{align}
Since this is upper-bounding the regular AROVR,
$\min_\theta \max_{z^-} \ell(z^-) - \tfrac{1}{\beta} U_\theta(z^-;\theta)$,
$d_\gamma$ tends to approach $\imath\{z=z^*\}$. If we maximize over
$\gamma$, then $d_\gamma(z,z^*)$ approaches zero, and therefore $z^-$ and
$z^*$ will be independent. Consequently, some form of hyper-parameter search
or dedicated meta-learning is required to obtain the optimal choice for $\gamma$.

\subsection{Choices for $d_\gamma$}
\label{sec:choices_for_d}

\paragraph{Weighted element-wise distances:}
The immediate choice for $d_\gamma$ is given by
\begin{align}
  d_\gamma(z,z') = \sum\nolimits_k \gamma_k \cdot \norm{z_k-z_k'}^2,
\end{align}
where $k$ ranges over the layers (or more generally, over the network units)
and $\gamma=(\gamma_1,\dotsc,\gamma_L)$ is a non-negative vector. Note that
this choice for $d_\gamma$ does not propagate perturbations across the
network: if the target loss $\ell$ does not depend on $z_k$, then maximization
over $z_k^-$ immediately leads to $z_k^-=z_k^*$. If the target loss
$\ell$ is e.g.\ a Euclidean squared error, $\ell(z)=\norm{z_L-y}^2$, then it
is easy to show that maximization w.r.t.\ $z^-$ solely scales the terms in
$\ell$,
\begin{align}
  \max\nolimits_{z_L^-}\, \norm{z_L^--y}^2 - \sum\nolimits_k \tfrac{\gamma_k}{\beta} \cdot \norm{z_k^--z_k^*}^2
  = \tfrac{\gamma_{L}}{\gamma_{L}-\beta} \norm{z_L^*-y}^2.
\end{align}
We therefore conclude that weighted element-wise distances (actually any
element-wise divergence) will exhibit limited generalization benefits.

\paragraph{Reweighted network potentials:}
We recall from Section~\ref{sec:background} that commonly used network
potentials $U_\theta$ decompose into terms over units. We define the
reweighted network potential
\begin{align}
  U_\theta^{\odot\gamma}(z;x) &\defeq \sum\nolimits_k \gamma_k \big( G_k(z_k) - z_k W_{k,:} z + G_k^*(W_{k,:} z) \big)
\end{align}
for a positive vector $\gamma$. Note that $U_\theta^{\odot\gamma}$ induces the
same network dynamics as $U_\theta$, and therefore
$U_\theta^{\odot\gamma}(z^*(x))=0$. We set
\begin{align}
  d_\gamma(z,z^*) = U_\theta^{\odot\gamma}(z;x),
\end{align}
which we insert into~\eqref{eq:relaxed_gen_AT} to obtain
\begin{align}
  \label{eq:weighted_potential_AT}
  \min\nolimits_\theta \max\nolimits_{z^-} \ell(z^-) - \tfrac{1}{\beta} U_\theta^{\odot\gamma}(z^-;x) .
\end{align}
This is of course conceptually the same as~\eqref{eq:relaxed_simple_AT}, but
the additional freedom of reweighting terms in $U_\theta$ implies that (1)
back-propagated error signal are scaled differently at each unit, and that (2)
the parameter $\gamma$ steers the anisotropy of the adversarial region. More
formally, an internal unit $z_k$ has the network dynamics according to
\begin{align}
  \begin{split}
    0 &= \gamma_k \Paren{ f_k^{-1}(z_k) - W_{k,:} z } + \sum\nolimits_{k'>k} \gamma_{k'} (W_{k',:})\tr\!\Paren{ f_{k'}(W_{k',:}z) - z_{k'} } \\
    \iff z_k &= f_k\!\Paren{ W_{k,:} z + \sum\nolimits_{k'>k} \tfrac{\gamma_{k'}}{\gamma_k} (W_{k',:})\tr\!\Paren{ z_{k'} - f_{k'}(W_{k',:}z) } }.
  \end{split}
\end{align}
Consequently, the ratio $\gamma_{k'}/\gamma_k$ weighs the error signal
$(W_{k',:})\tr(z_{k'} - f_{k'}(W_{k',:}z))$ propagated from upstream units
$k'>k$ differently. In the case of~\eqref{eq:weighted_potential_AT} the
error signal corresponds to the adversarial perturbation propagated backwards
from the output units. In order for~\eqref{eq:weighted_potential_AT} to have
a well-defined solution, $\gamma_k$ may be constrained for output units. Since
we focus on the 1-strongly convex least squares loss,
$\norm{z_L-y}^2/2$, $\gamma_L$ has to at least satisfy
$1-\gamma_L/\beta<0$ or $\gamma_L>\beta$.

\begin{remark}
  Scaling the terms $D_k(z_k\| a_k)$ in $U^{\odot\gamma}_\theta$ according
  to $\gamma$ is different from using layer- or neuron-specific learning
  rates. Purely linear networks (together with linearized target losses) are
  unaffected by the choice of $\gamma$. Otherwise the back-propagated error
  signal is amplified (or reduced) before applying the unit non-linearity. One
  way to interpret this is, that $\gamma$ steers the finite-difference
  spacing---something that is not possible in standard back-propagation which
  always uses the derivative $f_k'$ for the error signal.
\end{remark}

According to Proposition~\ref{prop:gen_AROVR} the best option to increase the
difficulty of the learning task (and therefore robustness of the resulting
network) is to choose the parameter vector $\gamma>0$ to be as small as
possible---or equivalently, to set $\beta$ as large as
possible. Figure~\ref{fig:lipschitz_betas}(top) illustrates this property in
the case of a $784\text{-}256^2\text{-}10$ ReLU network trained on a
5000-samples subset of Fashion-MNIST~\cite{xiao2017fashion_mnist}, but it also
demonstrates the substantial slowdown for large $\beta$ in terms of test
accuracy (based on early stopping using a validation
set). MNIST~\cite{lecun2010mnist} yields qualitatively similar
graphs. Surprisingly, in Fig.~\ref{fig:lipschitz_betas}(bottom) the behavior
is far less coherent when trained on a larger dataset (and the networks are
therefore outside the interpolation regime).
As discussed in the next section, there are also limits on how large
$\beta$ is allowed to be in order to ensure a finite solution for
$\min_z -\beta\ell(z) + U_\theta(z;x)$. In Section~\ref{sec:targeted} a
further generalization of AROVR is presented, that can remove these
restrictions on $\beta$.

\begin{figure}[tb]
  \includegraphics[width=0.5\textwidth]{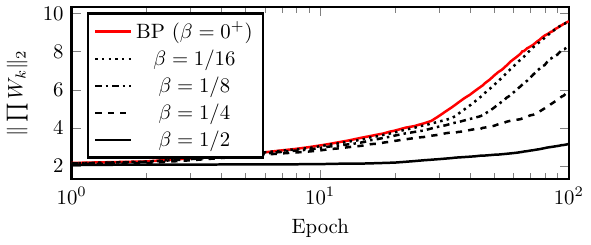}%
  \includegraphics[width=0.5\textwidth]{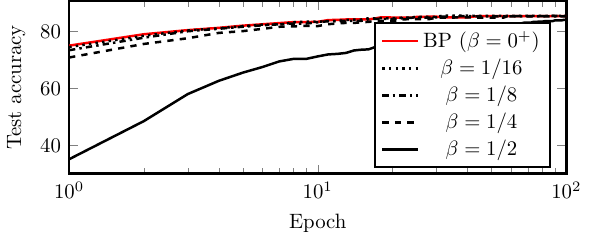}
  \includegraphics[width=0.5\textwidth]{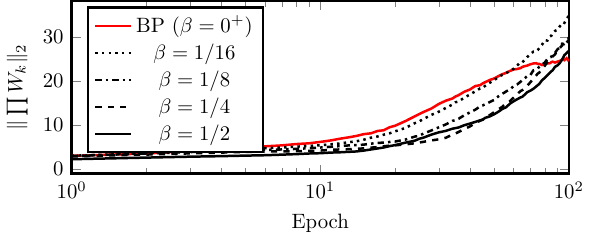}%
  \includegraphics[width=0.5\textwidth]{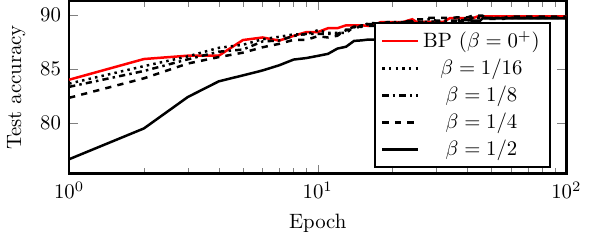}
  \caption{Evolution of Lipschitz estimates (left) and corresponding test
    accuracy (right, selected using 10\,000 validation samples). Top: graphs
    obtained using only 5000 samples from the training set. Bottom: training
    based on 50$\,$000 samples.}
  \label{fig:lipschitz_betas}
\end{figure}

Increasing the allowed adversarial perturbations at training time (either via
an increase of $\beta$ or by enlarging the allowed perturbations in other
adversarial training methods) can eventually be harmful for a network's
generalization ability. First, adversarial training can significantly slow
down the training progress (recall Fig.~\ref{fig:lipschitz_betas}(right)), and
second, an extremely robust classifier may exhibit both a large training error
as well as test error.

\subsection{Limitations of the (generalized) AROVR}

The differences between networks trained with back-propagation, ROVR and AROVR
are mostly observable in early training stages and diminish with larger
datasets or longer training times, in particular near and beyond the
interpolation regime.  The reason is that adversarially perturbed network
activities, $z^-\arg\min_z U_\theta(z;x)-\beta\ell(z)$, are close to the
unperturbed forward activities $z^*=\arg\min_z U_\theta(z;x)$ due to
$\ell'(z^*)\approx 0$ in the interpolation regime. Warm-starting AROVR-based
learning from network weights $\theta^*$ interpolating the training data
leaves the weights virtually unaffected. For linear (or linearized) loss
functions, it is straightforward to show the following:
\begin{proposition}
  \label{prop:interpolation}
  Let $\ell(z)=g\tr z$ and $\theta^*$ such that $\ell'(z^*)=0$, where
  $z^*=\arg\min_z U_{\theta^*}(z;x)$ are the neural activities assigned in the
  forward pass. Then
  \begin{align}
    z^* = z^- := \arg\min\nolimits_z \Braces{ U_{\theta^*}(z;x) - \ell(z) }
  \end{align}
  and therefore
  $\tfrac{\partial}{\partial\theta} \text{AROVR} = \tfrac{\partial}{\partial\theta} \Paren{ U_{\theta^*}(z^*;x) - U_{\theta^*}(z^-;x) } = 0$.
\end{proposition}
The proposition follows by observing that both $\ell'(z^*)=0$ and
$\nabla_z U_{\theta^*}(z^*;x)=0$. For general nonlinear (but differentiable)
losses $\ell$, $z^*$ is at least a stationary point of the mapping
$z\mapsto U_{\theta^*}(z;x)-\ell(z)$ and might be consistently returned as the
minimizer of this mapping (e.g.\ when $\ell$ is the Euclidean loss and
$z^-$ is determined by block-coordinate descent warm-started from the forward
pass).\footnote{Note that similar concerns are valid for standard adversarial
  training, in particular when the perturbed sample is obtained by a
  gradient-based attack,
  e.g.~\cite{kurakin2016adversarial,madry2017towards}.}

In summary, AROVR-based training leads to potentially different solutions
(compared to back-propagation or ROVR) by following a different trajectory in
parameter space, but does not otherwise discriminate between solutions. If
this is a good property (since it does ultimately not perturb the training
objective) or an undesired feature (because it does not prevent over-training)
is a difficult question. These remarks only hold in the interpolation regime,
when the network has sufficient capacity to achieve a (near-)zero training
loss.

Further, a more refined analysis on the constraints on $\beta$ reveals, that
$\beta$ is more constrained than by $\beta\in(0,1)$ for 1-strongly convex
$\ell$ and $D_k(z_k\|a_k)$ (w.r.t.\ $z_k$). Using the Schur complement on the
respective Hessian matrix, it can be shown that a safe choice for
$\beta$ satisfies
\begin{align}
  \beta < \big( 1+\sum\nolimits_{k=1}^{L-1} \prod\nolimits_{j=k}^{L-1} \norm{W_j\tr W_j}_2 \big)^{-1}
\end{align}
for linear networks, i.e.\ when $D_k(z_k\|a_k)=\norm{z_k-a_k}^2/2$. This can
be extended to weighted network potentials $U_\theta^{\cdot\gamma}$ with more
complicated expressions. Nonlinear networks are surprisingly difficult to
analyse due to $U_\theta$ being only block-convex (but not jointly convex in
$\{z_k\}$). It is not obvious whether even linear target loss functions
$\ell$ always lead to bounded activities for all choices of $\beta>0$.  This
is one motivation to consider numerically ``safer'' variants of AROVR in the
following section, that nevertheless keep the underlying spirit of network
robustness.

\section{Targeted adversarial perturbations}
\label{sec:targeted}

The AROVR is by construction based on an untargeted adversarial attack (since
the target loss is maximized regardless of the true or predicted label). Now
we consider a targeted adversarial training task,
\begin{align}
  \min\nolimits_\theta \ell(z^-) \quad \text{s.t. } z^- = \arg\min\nolimits_z \ell(z;y^-) \quad U_\theta(z^-;x) - \min\nolimits_z U_\theta(z;x) \le \delta,
\end{align}
where $y^-$ is an adversarial target. With our usual assumption that
$\min_z U_\theta(z)=0$ for all $\theta$, this slightly simplifies to
\begin{align}
  \min\nolimits_\theta \ell(z^-) \qquad \text{s.t. } z^- = \arg\min\nolimits_z \ell(z;y^-) \qquad U_\theta(z^-;x) \le \delta.
\end{align}
We may replace the targeted adversarial loss $\ell(z;y^-)$ with a general
adversarial loss $\ell^-(z)$, which e.g.\ also allows untargeted attacks via
the choice $\ell^-(z)=-\ell(z)$,
\begin{align}
  \min\nolimits_\theta \ell(z^-) \qquad \text{s.t. } \ell^-(z^-) \le \min_{z:U_\theta(z;x)\le\delta} \ell^-(z) \qquad U_\theta(z^-;x) \le \delta,
\end{align}
where we also rephrased that $z^-$ is a respective minimizer of
$\ell^-$ (subject to the adversarial region constraint). By using a multiplier
$1/\beta>0$ for the region constraint, we arrive at a generalization of~\eqref{eq:learning},
\begin{align}
  \begin{split}
    & \min\nolimits_\theta \ell(z^-) \qquad \text{s.t. } \beta\ell^-(z^-) + U_\theta(z^-;x) \le \min\nolimits_z \beta\ell^-(z) + U_\theta(z;x) \\
    \equiv\; & \min\nolimits_\theta \ell(z^-) \qquad \text{s.t. } z^- = \arg\min\nolimits_z \beta\ell^-(z) + U_\theta(z;x).
  \end{split}
  \label{eq:targeted_AT}
\end{align}
Under the assumption that $z^-$ is unique,~\eqref{eq:targeted_AT} can be
artifically extended to
\begin{align}
  \begin{split}
    \min\nolimits_\theta \min\nolimits_{z^-} \ell(z^-) \quad \text{s.t. } z^- &= \arg\min\nolimits_z \beta \ell^-(z) + U_\theta(z;x) \\
    \text{or}\qquad \min\nolimits_\theta \max\nolimits_{z^-} \ell(z^-) \quad \text{s.t. } z^- &= \arg\min\nolimits_z \beta \ell^-(z) + U_\theta(z;x).
  \end{split}
\end{align}
Applying the ROVR with multiplier $1/\beta'>0$ on the first variant yields the
relaxed problem,
\begin{align}
  \min_\theta \min_z \left\{ \ell(z) + \tfrac{\beta}{\beta'}\ell^-(z) + \tfrac{1}{\beta'} U_\theta(z;x) \right\} 
  - \min_z \left\{ \tfrac{\beta}{\beta'}\ell^-(z) + \tfrac{1}{\beta'} U_\theta(z;x) \right\},
\end{align}
and for the second variant we obtain
\begin{align}
  \min_\theta \min_z \left\{ \tfrac{\beta}{\beta'}\ell^-(z) + \tfrac{1}{\beta'} U_\theta(z;x) \right\}
  - \min_z \left\{ \tfrac{\beta}{\beta'}\ell^-(z) - \ell(z) + \tfrac{1}{\beta'} U_\theta(z;x) \right\}.
\end{align}
After dividing by $\beta/\beta'$, renaming
$\beta'/\beta \leadsto\alpha$ and scaling the adversarial loss
$\ell^-$ with $\bar\alpha := 1-\alpha$ (which will be convenient later), these
programs can be stated as follows,
\begin{align}
  &\min_\theta \min_z \left\{ \bar\alpha\ell^-(z) + \alpha\ell(z) + \tfrac{1}{\beta} U_\theta(z;x) \right\}
    - \min_z \left\{ \bar\alpha\ell^-(z) + \tfrac{1}{\beta} U_\theta(z;x) \right\} \tag{ROVR${}_{\alpha,\beta}$} \\
  &\min_\theta \min_z \left\{ \bar\alpha\ell^-(z) + \tfrac{1}{\beta} U_\theta(z;x) \right\}
    - \min_z \left\{ \bar\alpha\ell^-(z) - \alpha\ell(z) + \tfrac{1}{\beta} U_\theta(z;x) \right\}, \tag{AROVR${}_{\alpha,\beta}$}
\end{align}
and are of importance for the remainder of this section. Depending on the
context, ROVR${}_{\alpha,\beta}$ and AROVR${}_{\alpha,\beta}$ refer to the
entire learning task (minimization over parameters $\theta$) or just
the contrastive loss inside the outermost minimization.

In contrast to the untargeted setting (where $\ell^-(z)=-\ell(z)$), the
loss-related terms do generally not cancel in any of the above subproblems
w.r.t.\ $z$. In the case of the Euclidean loss
($\ell(z;y)=\norm{z-y}^2/2$) we deduce
\begin{align}
  \begin{split}
    \bar\alpha\ell^-(z) + \alpha\ell(z) &= \bar\alpha\ell(z;y^-) + \alpha\ell(z;y) \doteq \tfrac{1}{2} \Norm{z - (\bar\alpha y^- + \alpha y)}^2 \\
    \bar\alpha\ell^-(z) - \alpha\ell(z) &= \bar\alpha\ell(z;y^-) - \alpha\ell(z;y) \doteq \tfrac{1-2\alpha}{2} \norm{z}^2 - z\tr (\bar\alpha y'- \alpha y) \\
    &\doteq
      \begin{cases}
        \tfrac{1-2\alpha}{2} \norm{z - (\bar\alpha y'-\alpha y)}^2 & \text{if } \alpha<1/2 \\
        \tfrac{1}{2} z\tr (y-y') & \text{if } \alpha=1/2.
      \end{cases}
  \end{split}
\end{align}
Hence, $\bar\alpha\ell^-(z) + \alpha\ell(z)$ can be interpreted as label
smoothing, and $\bar\alpha\ell^-(z) - \alpha\ell(z)$ is the loss corresponding
to targeted adversarial attacks. If $\ell$ and $\ell^-$ are w.l.o.g.\
1-strongly convex and $\alpha<1/2$, then $\bar\alpha\ell^--\alpha\ell$ is
$(1\!-\!2\alpha)$-strongly convex, and both subproblems in
AROVR${}_{\alpha,\beta}$ have bounded solutions for any $\beta>0$. This is one
advantage of AROVR${}_{\alpha,\beta}$ over AROVR${}_{\beta}$.

The programs ROVR${}_{\alpha,\beta}$ and AROVR${}_{\alpha,\beta}$ clearly
reduce to ROVR${}_\beta$ and AROVR${}_\beta$, respectively, for the choice
$\alpha=1$. When $\alpha\in(0,1)$, then both programs facilitate learning
under perturbed network dynamics induced by a modified network potential
$\tilde U_\theta(z;x) = \bar\alpha\beta\ell^-(z) + U_\theta(z;x)$.
How are ROVR${}_{\alpha,\beta}$ and AROVR${}_{\alpha,\beta}$ related for
general $\alpha$? It is not very difficult to see that the learning task
defined in AROVR${}_{\alpha,\beta}$ is ``harder'' than the one corresponding
to ROVR${}_{\alpha,\beta}$:
\begin{proposition}
  \label{prop:ROVR_AROVR}
  ROVR${}_{\alpha,\beta}$ is a lower bound of AROVR${}_{\alpha,\beta}$.
\end{proposition}
All proofs are in the appendix.
Without further assumptions on $\ell$ and $\ell^-$ it is not possible to
relate e.g.\ AROVR (i.e.\ AROVR${}_{\alpha=1,\beta}$) and
AROVR${}_{\alpha,\beta}$ for general $\alpha$. The following simple
counterexamples further demonstrate that AROVR${}_{\alpha,\beta}$ can be both
an upper and a lower bound of $\ell(z^*)$ (with
$z^*=\arg\min_z U_\theta(z;x)$ being determined by the forward pass),
depending on the choice of $\ell^-$:
\begin{enumerate}
\item With $\ell^-\equiv 0$ or $\ell^-=-\ell$, AROVR${}_{\alpha,\beta}$ reduces to AROVR and is therefore an upper bound of $\ell(z^*)$.
\item With $\ell^-\equiv \frac{\alpha}{\bar \alpha} \ell$ we obtain ROVR, which is a lower bound of $\ell(z^*)$.
\end{enumerate}
In order for AROVR${}_{\alpha,\beta}$ to be an upper bound on
$\ell(z^*)$, it can be shown that $\ell^-$ needs to be ``aligned'' with
$-\ell$ such that $\ell(z^*)\le \ell(\tilde z)$, where
$\tilde z=\arg\min_z \beta\ell^-(z)+U_\theta(z;x)$ is the solution of the
perturbed dynamics.
If we change $\alpha$ and $\beta$ such that the product
$\bar\alpha\beta$ remains constant, then we have the following result:
\begin{proposition}
  \label{prop:AROVR_strengthening}
  Assume that $\ell(z)\ge 0$ for all $z$ and let
  $\alpha,\,\alpha' \in (0,1)$ and $\beta,\,\beta'>0$ be given such that
  $\alpha'\ge \alpha$ and $\bar\alpha\beta=\bar\alpha'\beta'$. Then
  \begin{align}
    \beta \text{AROVR}_{\alpha,\beta} \le \beta' \text{AROVR}_{\alpha',\beta'}.
  \end{align}
\end{proposition}
\begin{remark}
  After dividing both sides by
  $\rho:=\bar\alpha\beta=\bar\alpha'\beta'$ and renaming
  $\rho\leadsto\beta$, a different way to state the above result is as
  follows:
  \begin{align}
    \begin{split}
      &\min_z\Braces{ \ell^-(z) + \tfrac{1}{\beta} U_\theta(z;x) }
        - \min_z\Braces{ \ell^-(z) - \tfrac{\alpha}{1-\alpha}\ell(z) + \tfrac{1}{\beta} U_\theta(z;x) } \\
      &\le \min_z\Braces{ \ell^-(z) + \tfrac{1}{\beta} U_\theta(z;x) }
        - \min_z\Braces{ \ell^-(z) - \tfrac{\alpha'}{1-\alpha'}\ell(z) + \tfrac{1}{\beta} U_\theta(z;x) }.
    \end{split}
    \label{eq:AROVR_strengthening}
  \end{align}
  With $\alpha'\ge\alpha$ we have
  $\alpha/(1-\alpha)\ge \alpha'/(1-\alpha')$, and therefore the r.h.s.\
  of~\eqref{eq:AROVR_strengthening} is using a stronger (untargeted)
  adversarial pertubation than the one appearing in the l.h.s. This shall make
  the claim in Prop.~\ref{prop:AROVR_strengthening} more accessible. Without
  further assumptions on $\ell$ and $\ell^-$, the more desirable relation
  $\text{AROVR}_{\alpha,\beta} \le \text{AROVR}_{\alpha',\beta}$ whenever
  $\alpha'\ge \alpha$ does not hold.
\end{remark}

Instances of $\ell^-$ can be untargeted and targeted attacks as discussed
above, dropout-like regularization (see below) etc. If we let
$\beta\to 0^+$, then the following can be shown:
\begin{proposition}
  \label{prop:equivalence}
  Let $\ell$, $\ell^-$ and $U_\theta(\cdot;x)$ be twice continuously
  differentiable. Then
  \begin{align}
    \lim_{\beta\to 0^+} \tfrac{\partial}{\partial\theta} \text{ROVR}_{\alpha,\beta} = \lim_{\beta\to 0^+} \tfrac{\partial}{\partial\theta} \text{AROVR}_{\alpha,\beta}
    = \alpha \tfrac{d}{d\theta} \ell(z^*) \quad\text{s.t. } z^*=\arg\min_z U_\theta(z;x).
  \end{align}
\end{proposition}
This result can be obtained by implicit differentiation at
$\beta=0^+$. The perturbation induced by $\ell^-$ cancels, and only the target
loss $\ell$ remains (cf.\ appendix). This also implies that we
do not expect to adjust the trained network to compensate for training-time
perturbations, which is dissimilar to regularization techniques such as
dropout. Since infinitesimal $\beta$ (or very small $\beta$) behave
essentially like regular back-propagation, our focus is on the finite setting.

Interestingly, the proof of Proposition~\ref{prop:equivalence} reveals that
\emph{linear} perturbations of $U_\theta$ do in certain settings not affect
the above claim. Consequently we also consider replacing $\ell^-(z)$ with
$\ell^-(z) + \tfrac{1}{\beta} g\tr z$ for an arbitrary vector $g$ (that is
independent of $\beta$, $\theta$ and $z$). After introducing the shorthand
notation $U_\theta^g(z;x)=U_\theta(z;x)+g\tr z$, a further extension of AROVR
is therefore given by
\begin{align}
  \min_\theta \min_z \left\{ \bar\alpha\ell^-(z) + \tfrac{1}{\beta} U_\theta^g(z;x) \right\}
    - \min_z \left\{ \bar\alpha\ell^-(z) - \alpha\ell(z) + \tfrac{1}{\beta} U_\theta^g(z;x) \right\}. \tag{AROVR${}_{\alpha,\beta}^g$}
\end{align}
In this setting the influence of the perturbation vector $g$ does not vanish
with $\beta\to 0^+$, but has (to first order) the same effect on both
minimizers inside AROVR${}_{\alpha,\beta}^g$ (and analogously for
ROVR${}_{\alpha,\beta}^g$).
Of course, instead of using a fixed vector, $g$ can be randomly drawn from a
fixed distribution, and learning in that setting corresponds to a stochastic
optimization task,
\begin{align}
\min\nolimits_\theta \EE{g}{ \text{AROVR}_{\alpha,\beta}^g }.
\end{align}
If $g$ is a vector with elements randomly drawn from $\{0,M\}$ (where
$M>0$ is a sufficiently large constant), then we arrive at a dropout-like
regularization. Unlike proper dropout~\cite{srivastava2014dropout}, no
adaptation of the network parameters is needed (at least for infinitesimal
small $\beta$) at test time. This seems surprising but
AROVR${}_{\alpha,\beta}^g$ is different from standard dropout, e.g.\ setting a
neural activity to zero influences both upstream and downstream units in
contrast to regular dropout.

Finally, with $\alpha\to 0^+$, the two solutions of the subproblems inside AROVR${}_{\alpha,\beta}$,
\begin{align}
  \begin{split}
    z^+ &= \arg\min\nolimits_z \bar\alpha\beta\ell^-(z) + U_\theta(z;x) \\
    z^- &= \arg\min\nolimits_z \bar\alpha\beta\ell^-(z) - \alpha\beta\ell(z) + U_\theta(z;x),
  \end{split}
\end{align}
become increasingly similiar. In order to utilize the same learning rate in
the training process (and in light of Proposition~\ref{prop:equivalence}), the
actual loss function used in our implementation for gradient computation is
$\tfrac{1}{\alpha}\text{AROVR}_{\alpha,\beta}$.

\section{Empirical Observations}

\paragraph{Implementation:}
Since $z\mapsto U_\theta(z;x) + h(z)$ (where $h(z)$ is $\beta\ell(z)$,
$-\beta\ell(z)$ or $\bar\alpha\beta\ell^-(z)-\alpha\beta\ell(z)$) is
layer-wise convex, inferring $z$ is based on block coordinate descent by
successively updating $z_k$ in (in our implementation 20) backward passes
(after an initial forward pass). Since this inference procedure is not well
supported on GPUs (and we are using tiny datasets and networks), inference and
gradient-based learning is entirely CPU-based. One epoch of AROVR is
approximately $5\times$ slower than back-propagation.

\paragraph{Setup:}
The target loss $\ell$ is always the Euclidean loss due to closed-form
expressions for $z_L$. The attack loss $\ell^-$ is also the Euclidean loss for
any randomly sampled label different from the ground truth. In order to keep
$\bar\alpha\ell^-(z)-\alpha\ell(z)$ slightly strongly convex, we chose
$\alpha=0.49$. Given the results in Fig.~\ref{fig:lipschitz_betas}, the choice
$\beta=1/4$ is taken as a good compromise between convergence speed and the
ability to observe a difference in the resulting models. The perturbation
vector $g$ in AROVR${}_{\alpha,\beta}^g$ is element-wise sampled from
$\mathcal{N}(0,1/16)$ for each mini-batch. The Adam optimizer with a learning
rate of $2\cdot 10^{-4}$ is used throughout.

Table~\ref{tab:results} depicts the results of training a
768-$256^2$-10 ReLU network on MNIST and Fashion-MNIST, respectively,
accumulated over~3 random seeds and with and without back-propagation based
pretraining. In order to be in the interpolation regime, only 5000 (randomly
drawn) training samples are used. The reported test accuracy corresponds to
the model with highest validation accuracy (on a held-out set of 10\,000
samples). It is clear that AROVR does not escape the pretrained solution, in
accordance with Proposition~\ref{prop:interpolation}. Both
AROVR${}_{\alpha,\beta}$ and AROVR${}_{\alpha,\beta}^g$ reach similar
solutions regardless of pretraining. The improvements of the AROVR variants
are noticeable but smaller than anticipated, in particular for
Fashion-MNIST. One possible explanation might be that Fashion-MNIST is more
diverse, and adversarial training is not able to identify many regularities in
the dataset.

\begin{table}[tb]
  \setlength{\tabcolsep}{3pt}
  \footnotesize
  \begin{tabular}{l|c|ccc|ccc}
    \multicolumn{2}{c}{} & \multicolumn{3}{|c|}{w/o pretraining} & \multicolumn{3}{c}{w/ pretraining} \\ \hline
    & Backprop & AROVR & AROVR${}_{\alpha,\beta}$ & AROVR${}_{\alpha,\beta}^g$ & AROVR & AROVR${}_{\alpha,\beta}$ & AROVR${}_{\alpha,\beta}^g$ \\ \hline
    MNIST & $95.60\cpm0.13$ & $96.04\cpm0.15$ & $96.19\cpm0.13$ & $96.35\cpm0.06$ & $95.67\cpm0.14$ & $96.17\cpm0.10$ & $96.41\cpm0.10$ \\
    F-MNIST & $85.49\cpm0.09$ & $85.63\cpm0.32$ & $85.78\cpm0.10$ & $85.98\cpm0.21$ & $85.31\cpm0.19$ & $85.86\cpm0.33$ & $85.82\cpm0.25$
  \end{tabular}
  \caption{Test accuracy obtained by the various methods. See the text for details.}
  \label{tab:results}
\end{table}

\section{Discussion}

In the initial stages of this work we expected that robustness w.r.t.\
internal neurons is connected to some form of feature invariance and will lead
to significantly better generalization ability, in particular when training
with very small datasets. In practice, the observed improvements are rather
minor, and whether this is due to the simplistic network architecture,
limitations of the training process, or due to other factors is subject of
future research. So far we have shown certain theoretical guarantees of making
the training loss more difficult without directly interfering with the network
parameters or strongly biasing the obtained model (unlike many standard
regularization techniques such as weight decay).
Biological intelligence is remarkable in its intrinsic robustness and being
able to detect invariances from rather limited datasets---properties that are
largely absent in current deep learning methods. Hence, we believe that this
work is a small step towards a more economical use of training data.

A completely different line of future research is on the relation between
simple but scalable adversarial
training~\cite{kurakin2016adversarial,madry2017towards}, certified
robustness~\cite{wong2018provable,raghunathan2018semidefinite} and AROVR.

\bibliographystyle{splncs04}
\bibliography{main}


\appendix

\section{Proofs of the propositions}

\begin{proof}[Proposition~\ref{prop:ROVR_AROVR}]
  We show that for all $\theta$ it holds that
  \begin{align}
    \begin{split}
      &\min_z \left\{ \bar\alpha\ell^-(z) + \alpha\ell(z) + \tfrac{1}{\beta} U_\theta(z;x) \right\}
      - \min_z \left\{ \bar\alpha\ell^-(z) + \tfrac{1}{\beta} U_\theta(z;x) \right\} \\
      &\le \min_z \left\{ \bar\alpha\ell^-(z) + \tfrac{1}{\beta} U_\theta(z;x) \right\}
      - \min_z \left\{ \bar\alpha\ell^-(z) - \alpha\ell(z) + \tfrac{1}{\beta} U_\theta(z;x) \right\}.
    \end{split}
  \end{align}
  Let
  \begin{align}
    \tilde z := \arg\min\nolimits_z \bar\alpha\ell^-(z) + \tfrac{1}{\beta} U_\theta(z;x)
  \end{align}
  be the solution of the perturbed neural network dynamics. Consequently,
  \begin{align}
    \begin{split}
      \text{l.h.s.} &= \min_z \left\{ \bar\alpha\ell^-(z) + \alpha\ell(z) + \tfrac{1}{\beta} U_\theta(z;x) \right\}
      - \bar\alpha\ell^-(\tilde z) - \tfrac{1}{\beta} U_\theta(\tilde z;x) \\
      {} &\le \bar\alpha\ell^-(\tilde z) + \alpha\ell(\tilde z) + \tfrac{1}{\beta} U_\theta(\tilde z;x)
      - \bar\alpha\ell^-(\tilde z) - \tfrac{1}{\beta} U_\theta(\tilde z;x) \\
      {} &= \alpha\ell(\tilde z) \\
      {} &= \bar\alpha\ell^-(\tilde z) + \tfrac{1}{\beta} U_\theta(\tilde z;x) + \alpha\ell(\tilde z)
      - \bar\alpha\ell^-(\tilde z) - \tfrac{1}{\beta} U_\theta(\tilde z;x) \\
      {} &\le \bar\alpha\ell^-(\tilde z) + \tfrac{1}{\beta} U_\theta(\tilde z;x)
      + \max_z\left\{ \alpha\ell(z) - \bar\alpha\ell^-(z) - \tfrac{1}{\beta} U_\theta(z;x) \right\} \\
      {} &= \bar\alpha\ell^-(\tilde z) + \tfrac{1}{\beta} U_\theta(\tilde z;x)
      - \min_z \left\{ -\alpha\ell(z) + \bar\alpha\ell^-(z) + \tfrac{1}{\beta} U_\theta(z;x) \right\} \\
      &= \text{r.h.s.,}
    \end{split}
  \end{align}
  which completes the proof. \qed
\end{proof}

\begin{proof}[Proposition~\ref{prop:AROVR_strengthening}]
  We denote $\rho:=\bar\alpha\beta = \bar\alpha'\beta'$ and abbreviate
  $U_\theta(z;x)$ as $U(z)$. First, we observe that
  $\bar\alpha'\le \bar\alpha$ and $\beta\le\beta'$ as
  \begin{align}
    \beta = \frac{\rho}{\bar\alpha} \le \frac{\rho}{\bar\alpha'} = \beta'.
  \end{align}
  We now have to show that
  \begin{align}
    \begin{split}
      &\min_z\Braces{ \rho\ell^-(z) + U(z) } - \min_z\Braces{ \rho\ell^-(z) - \alpha\beta\ell(z) + U(z) } \\
      &\le \min_z\Braces{ \rho\ell^-(z) + U(z) } - \min_z\Braces{ \rho\ell^-(z) - \alpha'\beta'\ell(z) + U(z) }
    \end{split}
  \end{align}
  or
  \begin{align}
    \min_z\Braces{ \rho\ell^-(z) - \alpha'\beta'\ell(z) + U(z) } \le \min_z\Braces{ \rho\ell^-(z) - \alpha\beta\ell(z) + U(z) }.
  \end{align}
  Let $z^-$ be the minimizer of the r.h.s. We have the chain of inequalities,
  \begin{align}
    \begin{split}
      \text{l.h.s.} &= \min_z\Braces{ \rho\ell^-(z) - \alpha'\beta'\ell(z) + U(z) } \\
      &\le \rho\ell^-(z^-) - \alpha'\beta'\ell(z^-) + U(z^-) \\
      &\le \rho\ell^-(z^-) - \alpha\beta\ell(z^-) + U(z^-) = \text{r.h.s.},
    \end{split}
  \end{align}
  where we made use of
  \begin{align}
    \alpha\beta - \alpha'\beta' = (1-\bar\alpha)\beta - (1-\bar\alpha')\beta' = \beta - \rho - \beta' + \rho = \beta -\beta' \le 0
  \end{align}
  and consequently $(\alpha\beta - \alpha'\beta')\ell(z^-) \le 0$. \qed
\end{proof}

\begin{proof}[Proposition~\ref{prop:equivalence}]
  We focus on AROVR${}_{\alpha,\beta}$, since ROVR${}_{\alpha,\beta}$ works
  analogously. For brevity we write $U(z)$ for $U_\theta(z;x)$. Recall that
  derivatives of functions $f:\mathbb{R}^n\to\mathbb{R}^m$ are
  $m\times n$ matrices, and therefore
  $\nabla f(z) = \frac{d}{dz}f(z)\tr$ for real-valued functions $f$. Now the
  solutions of the respective subproblems inside
  AROVR${}_{\alpha,\beta}$ are denoted as
  \begin{align}
    \begin{split}
      z^+ &= \arg\min\nolimits_z \bar\alpha\beta\ell^-(z) + U(z) \\
      z^- &= \arg\min\nolimits_z \bar\alpha\beta\ell^-(z) - \alpha\beta\ell(z) + U(z).
    \end{split}
  \end{align}
  Consequently, $z^+$ and $z^-$ satisfy
  \begin{align}
    \begin{split}
      0 &= \bar\alpha\beta\nabla\ell^-(z^+) + \nabla U(z^+) \\
      0 &= \bar\alpha\beta\nabla\ell^-(z^-) - \alpha\beta\nabla\ell(z^-) + \nabla U(z^-).
    \end{split}
  \end{align}
  Both $z^+$ and $z^-$ depend on $\beta$, hence implicit differentiation w.r.t.\ $\beta$
  yields
  \begin{align}
    \begin{split}
      0 &= \Paren{ \bar\alpha\beta\nabla^2\ell^-(z^+) + \nabla^2 U(z^+) } \Paren{\frac{dz^+}{d\beta}}\tr + \bar\alpha\nabla\ell^-(z^+) \\
      0 &= \Paren{ \bar\alpha\beta\nabla^2\ell^-(z^-) - \alpha\beta\nabla\ell(z^-) + \nabla^2 U(z^-) } \Paren{\frac{dz^-}{d\beta}}\tr
          + \bar\alpha\nabla\ell^-(z^-) - \alpha\nabla\ell(z^-),
    \end{split}
  \end{align}
  from which we deduce that
  \begin{align}
    \begin{split}
      \frac{dz^+}{d\beta}\big|_{\beta=0^+} &= -\bar\alpha \nabla\ell^-(z^+)\tr \Paren{ \nabla^2 U(z^+)}^{-1} \\
      \frac{dz^-}{d\beta}\big|_{\beta=0^+} &= -\Paren{ \bar\alpha\nabla\ell^-(z^-) - \alpha\nabla\ell(z^-) }\tr \Paren{ \nabla^2 U(z^+)}^{-1}.
    \end{split}
    \label{eq:implicit_diff}
  \end{align}
  As $\lim_{\beta\to 0^+} z^+=\lim_{\beta\to 0^+} z^-=z^*$, where
  $z^*=\arg\min_z U(z)$, those relations simplify to
  \begin{align}
    \begin{split}
      \frac{dz^+}{d\beta}\big|_{\beta=0^+} &= -\bar\alpha\Paren{ \nabla^2 U(z^*)}^{-1} \nabla\ell^-(z^*) \\
      \frac{dz^-}{d\beta}\big|_{\beta=0^+} &= -\Paren{ \nabla^2 U(z^*)}^{-1} \Paren{ \bar\alpha\nabla\ell^-(z^*) - \alpha\nabla\ell(z^*) }.
    \end{split}
  \end{align}
  Consequently,
  \begin{align}
    \frac{dz^+}{d\beta}\big|_{\beta=0^+} - \frac{dz^-}{d\beta}\big|_{\beta=0^+} = -\alpha \nabla\ell(z^*)\tr \Paren{ \nabla^2 U(z^*)}^{-1}.
  \end{align}
  Note that these relations still hold if $U$ is perturbed by a linear term
  and $\nabla^2 U$ is (locally) constant, i.e.
  \begin{align}
    \begin{split}
      z^+ &= \arg\min\nolimits_z \bar\alpha\beta\ell^-(z) + U(z) + g\tr z \\
      z^- &= \arg\min\nolimits_z \bar\alpha\beta\ell^-(z) - \alpha\beta\ell(z) + U(z) + g\tr z.
    \end{split}
  \end{align}
  and therefore
  \begin{align}
    \begin{split}
      0 &= \bar\alpha\beta\nabla\ell^-(z^+) + \nabla U(z^+) + g \\
      0 &= \bar\alpha\beta\nabla\ell^-(z^-) - \alpha\beta\nabla\ell(z^-) + \nabla U(z^-) + g.
    \end{split}
  \end{align}
  Since $g$ does not depend on $\beta$ or $z$, it will vanish in the implicit
  differentian step and obtain the same relations as
  in~\eqref{eq:implicit_diff}. Continuing with the proof we read
  \begin{align}
    \begin{split}
      \lim_{\beta\to 0^+} \partfrac{}{\theta} \text{AROVR}_{\alpha,\beta}
      &= \lim_{\beta\to 0^+} \partfrac{}{\theta} \tfrac{1}{\beta} \Paren{ U_\theta(z^+;x) - U_\theta(z^-;x) } \\
      &= \lim_{\beta\to 0^+} \tfrac{1}{\beta} \Paren{ \partfrac{}{\theta} U_\theta(z^+;x) - \partfrac{}{\theta} U_\theta(z^-;x) } \\
      &= \tfrac{d}{d\beta} \Paren{ \partfrac{}{\theta} U_\theta(z^+;x) - \partfrac{}{\theta} U_\theta(z^-;x) }\big|_{\beta=0^+} \\
      &= \tfrac{dz^+}{d\beta}\big|_{\beta=0^+} \partfrac{{}^2}{z\partial\theta} U_\theta(z^*;x)
        - \tfrac{dz^-}{d\beta}\big|_{\beta=0^+} \partfrac{{}^2}{z\partial\theta} U_\theta(z^*;x) \\
      &= \Paren{ \tfrac{dz^+}{d\beta}\big|_{\beta=0^+} - \tfrac{dz^-}{d\beta}\big|_{\beta=0^+} } \partfrac{{}^2}{z\partial\theta} U_\theta(z^*;x) \\
      &= -\alpha \nabla\ell(z^*)\tr \Paren{\nabla_z^2 U_\theta(z^*;x)}^{-1} \tfrac{\partial^2}{\partial z\partial\theta} U_\theta(z^*;x).
    \end{split}
    \label{eq:implicit_deriv}
  \end{align}
  On the other hand, $z^*=\arg\min_z U_\theta(z;x)$ satisfies
  $\nabla_z U_\theta(z^*;x)=0$, from which we deduce (via implicit
  differentiation) that
  \begin{align}
    \frac{dz^*}{d\theta} &= -\Paren{\nabla_z^2 U_\theta(z^*;x)}^{-1} \tfrac{\partial^2}{\partial z\partial\theta} U_\theta(z^*;x).
  \end{align}
  Hence, via the chain rule we obtain
  \begin{align}
    \begin{split}
      \tfrac{d}{d\theta} \ell(z^*) &= \nabla\ell(z^*)\tr \tfrac{dz^*}{d\theta} \\
      &= -\nabla\ell(z^*)\tr \Paren{\nabla_z^2 U_\theta(z^*;x)}^{-1} \tfrac{\partial^2}{\partial z\partial\theta} U_\theta(z^*;x).
    \end{split}
  \end{align}
  The claim is obtain by comparing this with~\eqref{eq:implicit_deriv}. We
  need $U_\theta$ to be twice continuous differentiable to ensure that
  $\nabla_z^2$ is symmetric. \qed
\end{proof}

\end{document}